\documentclass[twoside]{article}
\usepackage{style}
\usepackage[toc,page]{appendix}
\usepackage[mathscr]{euscript}
\usepackage[ruled,vlined]{algorithm2e}
\usepackage{subfigure}
\usepackage{bbm}
\usepackage{todonotes}
\usepackage{natbib}

\DeclareSymbolFont{rsfs}{U}{rsfs}{m}{n}
\DeclareSymbolFontAlphabet{\mathscrsfs}{rsfs}

\usepackage{hyperref}
\hypersetup{
    colorlinks,
    linkcolor={blue!80!black},
    citecolor={green!50!black},
    urlcolor={black},
}
\colorlet{linkequation}{blue}
\definecolor{zy}{RGB}{225,0,0}

\usepackage[accepted]{aistats2023}
\begin{document}

\twocolumn[

\aistatstitle{Learning Correlated Stackelberg Equilibrium in General-Sum Multi-Leader-Single-Follower Games}

\aistatsauthor{ Yaolong Yu$^{\dagger}$ \And Haifeng Xu$^{\ddagger}$ \And  Haipeng Chen$^{\diamond}$ }

\aistatsaddress{ Shandong University \And University of Chicago \And  College of William \& Mary } ]

\begin{abstract}
Many real-world strategic games involve interactions between multiple players. We study a hierarchical multi-player game structure, where players with asymmetric roles can be separated into leaders and followers, a setting often referred to as Stackelberg game or leader-follower game. In particular, we focus on a Stackelberg game scenario where there are multiple leaders and a single follower, called the Multi-Leader-Single-Follower (MLSF) game. 
We propose a novel asymmetric equilibrium concept for the MLSF game called Correlated Stackelberg Equilibrium (CSE). We design online learning algorithms that enable the players to interact in a distributed manner, and prove that it can achieve no-external Stackelberg-regret learning. This further translates to the convergence to approximate CSE via a reduction from no-external regret to no-swap regret. At the core of our works, we solve the intricate problem of how to learn equilibrium in leader-follower games with noisy bandit feedback by balancing exploration and exploitation in different learning structures. 
\end{abstract}

\section{Introduction}

Game theory studies the interactions between multiple strategic players or agents~\citep{roughgarden2010algorithmic,osborne2004introduction}. Many real-world domains such as economics and policy making can be described using a hierarchical game structure among the players, where the two levels of players have asymmetric roles and can be partitioned into \textit{leaders} and \textit{followers}~\citep{sherali1983stackelberg}. 
This type of games is called \emph{Stackelberg} or leader-follower games. Depending on the game structures, the Stackelberg game literature can be categorized into single-leader-single-follower (SLSF) games \citep{conitzer2006computing,blum2014learning}, single-leader-multi-follower (SLMF) games~\citep{ramos2016water,salas2020optimal}, multi-leader-single-follower (MLSF) games ~\citep{aussel2016deregulated,escobar2008equilibrium,hu2007using,gan2018stackelberg}, and multi-leader-multi-follower (MLMF) games~\citep{mallozzi2017multi,sherali1983stackelberg}. 
We focus on the repeated general-sum MLSF game setting, where at each round, multiple leaders first make a collective decision, and then a single follower reacts to the leaders' decision. Such a game setting has broad implications in real-world problems, such as security games~\citep{gan2018stackelberg}, deregulated electricity markets \citep{aussel2016deregulated}, and industrial eco-parks~\citep{ramos2016water}. Most of the MLSF games literature assume that the game parameters (e.g., players' loss functions) are known a-priori, and they focus on finding equilibrium in MLSF games via optimization methods~\citep{leyffer2010solving,kulkarni2014shared}. In this paper, we present the first study of MLSF games from a \textit{learning} perspective. In particular, we ask the fundamental research question:

\textit{Can we design efficient learning algorithms that provably reach equilibrium in repeated general-sum MLSF games?}

An immediate follow-up question is: what is an appropriate equilibrium concept for this setting? \citet{gan2018stackelberg} propose an equilibrium concept called Nash stackelberg equilibrium (NSE) for multi-defender-single-attacker security games. But as pointed out in \citet{gan2018stackelberg}, when there exist malicious defenders, an $\epsilon$-NSE may not exist for every $\epsilon>0$, which means that NSE does not always exist in general-sum MLSF games. Inspired by this and the correlated equilibrium concept~\citep{aumann1974subjectivity,aumann1987correlated}, we propose a more viable and realistic equilibrium concept, which we call Correlated Stackelberg Equilibrium (CSE).

Using CSE as the equilibrium concept, we give an affirmative answer to the above question for a broad range of repeated general-sum MLSF games. We summarize our key contributions as follows: 
(i) We start with a simpler setting where each leader knows the loss functions of the follower and itself, and prove that classical adversarial online learning algorithms like Hedge~\citep{cesa2006prediction,littlestone1994weighted} can learn to reach approximate CSE. The result also holds when it is extended to a slightly more complicated setting where the leaders do not know the loss function of the follower, but have access to an oracle that returns the best response of the follower.
(ii) Building on the insights from the simpler settings, we then study the more challenging setting of MLSF games with \textit{noisy bandit feedback}. We design a distributed learning algorithm (called $\alpha$EXP3-UCB) for leaders and the follower. We first study a degenerate scenario of MLSF games, which is the SLSF games setting with noisy bandit feedback (a special case of CSE with a single leader), and prove that it converges to the Stackelberg equilibrium using the $\alpha$Exp3-UCB algorithm. This result is a non-trivial improvement of \citet{bai2021sample} in the online learning setting. (iii) In the ultimate setting of MLSF games, we first provide complexity analysis that shows the hardness of the $\alpha$Exp3-UCB algorithm in a true MLSF setting. Based on the analysis, we then devise a more efficient two-stage learning algorithm that still provably learns to converge to CSE in the MLSF setting.

\section{Related Work}
\paragraph{Leader-follower games}
There is a line of works that studies optimization-based methods for finding equilibrium in MLSF games~\citep{leyffer2010solving,kulkarni2014shared,vicente1994bilevel}. They center around a bilevel optimization problem structure called equilibrium program with equilibrium constraints (EPEC), and aim at devising efficient optimization methods to solve the bilevel optimization problem. 
As important application domains, many works study the specific MLSF scenarios of deregulated electricity markets \citep{allevi2018equilibrium,escobar2008equilibrium}, or multi-denfender security game~\citep{jiang2013defender,gan2018stackelberg,basilico2017computing}. 
\citet{gan2018stackelberg} propose NSE as an equilibrium concept for the multi-defender security games. They show that an exact equilibrium may fail to exist, and deciding whether it exists is NP-hard. Moreover, they show that an approximate $\epsilon$-NSE may not exist for every $\epsilon>0$ in the presence of malicious defenders. This motivates us to come up with CSE as a more viable equilibrium concept.


\paragraph{Learning Stackelberg equilibrium in games}
The line of works that is close to our work is the literature on learning Stackelberg equilibrium in SLSF games.
\citet{peng2019learning} study the problem of
learning the optimal leader strategy in Stackelberg games with a follower best response oracle, which means that the leader knows the exact follower best response to its action. Such kind of assumption is also made in other works along that line~\citet{letchford2009learning,blum2014learning}.
Two player zero-sum games have been extensively studied in the broader games and learning literature~\citep{fasoulakis2021forward,rakhlin2013optimization,balduzzi2019open}, and it is known that Stackelberg equilibrium is equivalent to Nash equilibrium in convex-concave setting due to von Neumann's minimax theorem~\citep{v1928theorie}.
The equilibrium in two player zero-sum game is easier to compute because the two players essentially have the same objective for any action pair. This property no longer exists for general-sum Stackelberg games, and therefore general-sum Stackelberg games are considered harder to learn. 
\citet{bai2021sample} study learning algorithms for Starkelberg equilibrium in general-sum SLSF games with noisy bandit feedback in a batch version. Their proposed algorithm needs to query each pair of leader-follower actions for sufficient rounds to calculate the empirical mean. Therefore, it  requires a \emph{centralized} authority to learn the equilibrium instead of distributed player self-learning, as is in this paper. Moreover, we focus on the more challenging problem of learning CSE in general-sum MLSF games with noisy bandit feedback, which is much more computationally expensive because of the exponentially growing joint action space for multiple leaders. To the best of our knowledge, this is the first work on MLSF games of any kind from a learning perspective.  

\section{Preliminary}

\subsection{Repeated general-sum MLSF games} 
A general-sum MLSF game is represented as a tuple $\{\mathcal{A}, \mathcal{B}, l\}$. In this setting, two levels of decision makers are considered: a set of $m$ leaders $1,\ldots,m$ and one follower $f$. $\mathcal{A}_i$ represents the action set of leader $i$, and $\mathcal{B}$ represents the action set of follower $f$. For each leader $i$, $|\mathcal{A}_i|=n_i$ is the cardinality of its action set. We assume $n_1=n_2=\cdots=n_m=n$ for clarity, and all the results can be generalized when they are not the same. We denote $\mathcal{A}=\mathcal{A}_1 \times \mathcal{A}_2 \times \cdots \times \mathcal{A}_m$ the action set for all the leaders, and $\mathcal{A}\times\mathcal{B}$ the joint action set of all the leaders and the follower. We assume all action sets are discrete and finite. For any joint action $(a,b)=(a_1,a_2,\cdots,a_n,b)\in \mathcal{A}\times \mathcal{B}$, the loss function for leader $i$ is $l_{i}=\{l_i(a,b): \mathcal{A}\times\mathcal{B}\rightarrow [0,1]\}$. To distinguish the loss function and the noisy (i.e., stochastic) loss value for a given data sample, we use a different notation $\xi_i(a,b)\in [0,1]$ to represent the noisy form loss value for leader $i$. By definition, $l_i(a, b)=\mathbb{E}[\xi_i(a, b)]$.
Similarly for follower $f$, the loss function and the noisy loss value for one data sample are respectively represented as $l_f=\{l_f(a,b): \mathcal{A}\times\mathcal{B}\rightarrow [0,1]\}$ and $\xi_f(a,b)\in [0,1]$, where $l_f(a, b)=\mathbb{E}[\xi_f(a, b)]$.  
For each leader and the follower, the goal is to minimize its own loss function. 

In a repeated game, the players play iteratively at each round $t$, with a time horizon of $T$ rounds. Because of the asynchronous moves, the leader usually maintains a \textit{mixed strategy}. 
We denote leader $i$'s \textit{mixed strategy} as $P_i=\{P_i(a_i):\mathcal{A}_i\rightarrow [0,1],\sum_{a_i\in \mathcal{A}_i}P_i(a_i)=1\}$, which is a probability distribution of taking each action $a_i\in\mathcal{A}_i$. $P_i^t$ is the mixed strategy of leader $i$ at round $t$. $\chi^t=P_1^t\times\cdots\times P_m^t$ is the joint mixed strategy for all the leaders at round $t$.  $\Bar{\chi} = \frac{1}{T}\sum_{t=1}^T \chi^t$ is the time averaged joint strategy profile distribution for the all the $T$ rounds. After the leaders commit a certain joint action from their mixed strategies, the follower then observes the action and uses a \textit{pure strategy} $b$.
In each round of a typical repeated general-sum MLSF game, the procedure is as follows:
\begin{itemize}
    \item
    Each leader $i$ simultaneously plays an action $a_i \in \mathcal{A}_i$ according to their mixed strategies.
    \item
    The follower $f$ observes the leaders' joint action $a$, and plays a pure strategy $b$ as the response.
    \item
    Depending on the game setting (the type of feedback), each leader $i$ observes the loss $l_i(a,b)$ or noisy loss $\xi_i(a,b)$,  while the follower $f$ observes its loss $l_f(a,b)$ or noisy loss $\xi_f(a,b)$.
\end{itemize}

\begin{assumption}\label{assum:singleton}
For simplicity, we assume that for every action $a\in \mathcal{A}$, the set of follower best responses $\text{Br}(a)=\argmin_{b\in \mathcal{B}} l_{f}(a,b)$ is a singleton, i.e., $|\text{Br}(a)|=1$. This means that for any action played by leaders, the follower has a unique best response.
\end{assumption}

\textbf{Remark.}
There has been ambiguity in the equilibrium formulation in Stackelberg games when there may be multiple follower best responses, since there is no explicit ways of breaking ties between multiple follower best responses. A typical way for tie-breaking is to assume that the follower is either in favor of the leader utility (optimistic) or against it (pessimistic). However, we found that this tie-breaking rule incurs further ambiguity in the MLSF games setting because if each leader assumes that follower is in favor of (or against) the leader itself, then the leaders are assuming different follower best responses (one for each leader). This inconsistency makes it infeasible to converge to any equilibrium in a repeated MLSF games setting. Therefore, we follow the common practice in similar problems such as multi-armed bandit games~\citep{auer2002nonstochastic,audibert2010best} and MLSF games~\citep{aussel2020short} and assume a unique best response. 
Since equilibrium concepts like NSE~\citep{gan2020decentralized} do not generally hold under wild conditions, we propose a more viable equilibrium concept as follows: 
\begin{definition}[$\epsilon$-CSE]\label{def:cse}
Define swap function $s: \mathcal{A}_i \to \mathcal{A}_i$. $\chi$ is an $\epsilon$-CSE when the following inequality holds for any swap function $s$, for any $i\in[m]$,
\begin{equation}
\!\mathop{\mathbb{E}}\limits_{a \sim \chi}\!\!\!\left[l_{i}(a,\text{Br}(a)) \right] \!\leq\!\!\! \mathop{\mathbb{E}}\limits_{a \sim \chi}\!\!\left[l_{i}\!\left(s(a_{i}),\!a_{-i},\!\text{Br}(s(a_{i}),a_{-i})\right) \right]\! + \epsilon
\end{equation}
\end{definition}
In the rest of the paper, we will present existence results of CSE under a broad range of MLSF game settings. Apart from the existence property, CSE is also more desirable because it does not require centralized learning. For example, in many real-world security domains, there are often multiple self-interested security teams who conduct patrols over the same set of important targets without coordinating with each other~\citep{jiang2013defender}. As another example of wildlife conservation, when patrol teams from various NGOs or provinces patrol within the same conservation area to protect wildlife from poaching -- different NGOs or provinces may have different types of targeted species (e.g., the situation in Pakistan~\citep{wwf2015parkistan}) and tend to operate separately. Similarly, there are different countries which simultaneously plan their own anti-crime actions in international waters against illegal fishing~\citep{klein2017can}.

\subsection{Learning CSE}
In practice, the ground truth loss functions $l_i, \forall i\in [m]$ and $l_f$ are not revealed to each other, or even not known by the player itself, so that the players need to learn via repeatedly playing the games. This results in an online learning setting, where the players need to balance between \textit{exploiting} the optimal strategy learned so far, and \textit{exploring} potential better strategies.

We define ${\Delta_{ak}}=l_f(a,k)-\min_{b \in \mathcal{B}} l_f(a,b)$ for any $a\in \mathcal{A}$, $k\in[n_f]$ to measure the gap in losses between the follower's k-th action and the optimal one when leaders play $a$. And we define minimal loss gap as $\varepsilon=\min_{a\in\mathcal{A},k\in[n_f]} \Delta_{ak}$. The objective of follower is to find the best response $\text{Br}(a)$ for every joint action $a \in \mathcal{A}$.

The expected loss of a leader $i$'s action under the follower best response is as follows. Note that the expectation is taken over the joint mixed strategy distribution of the other leaders, i.e.,
\begin{equation}\label{eq:expected_loss_full}
L_i^t(a_i)=\mathbb{E}_{a_{-i}^t\sim P_{-i}^t}\left[l_i(a_i,a_{-i}^t,\text{Br}(a_i,a_{-i}^t))\right].
\end{equation}
The objective of each leader $i$ is to minimize the following Stackelberg-regret:
\begin{equation}\label{eq:regret}
R_i^S(T)=\sum_{t=1}^T\mathbb{E}_{a_i^t\sim P_i^t}\left[L_i^t(a_i^t)-L_i^t(a_{i,\star})\right],
\end{equation}
where $a_{i,\star}=\argmin_{a_i\in\mathcal{A}}\sum_{t=1}^TL_i^t(a_i)$ is the optimal action in hindsight. The superscript $S$ implies ``Stackelberg''. 



Our goal is to design online learning algorithms for both leaders and the follower that are able to achieve the $\epsilon$-CSE.
In the following context, we consider different settings in general-sum MLSF games. Depending on the types of feedback information on a player's loss function, we separate them into three categories, including full information, semi-bandit information and noisy bandit feedback. In the first two simpler settings, leaders have access to the follower's best response oracle. In the last and more challenging setting (where our main results are focused), all players can only get noisy bandit feedback.

\section{Warm-up:  follower best response  as an oracle }

We begin with two simple settings as warm up. In both cases, the leaders have access to an oracle which returns the follower's best response given the leaders' joint action. Therefore, we can essentially treat the follower as part of the environment that affects the leaders' loss values (the $\text{Br}(a_i,a_{-i}^t)$ term in Eq.\eqref{eq:expected_loss_full}) via the best response oracle.



\subsection{Existence of CSE with full information}\label{sec:warm-up-full}
In the full information setting, each leader $i$ knows its own loss function $l_i$ and the follower's loss function $l_f$, and can observe the joint mixed strategy $\chi^t$ at each round $t$. Note that although we call it ``full information'', leader $i$ does not need to know the other leaders' loss functions, so that the process is still distributed.

Hedge is an online learning algorithm first designed for solving how to learn from expert advice, aiming to minimize the expected cummulative losses in an adversarial environment~\citep{littlestone1994weighted,cesa2006prediction}. By applying it into full information MLSF games, we can prove that: 

\begin{proposition}
If leader $i$ uses the \text{Hedge} algorithm
in full information repeated general-sum MLSF games, define $R_i^S(T)$ as in Eq.\eqref{eq:regret}, it achieves no-external Stackelberg-regret in the following sense: 
$$
    R_i^S(T)\leq\mathcal{O}\left(\sqrt{\ln n T}\right).
$$
\end{proposition}

\text{Hedge} can achieve no-external regret learning against adversarial losses at each round by using a classical exponential weights update of the policies. At each round $t$, after all leaders show their mixed strategy distribution, each leader $i$ uses the \text{Hedge} algorithm to update its mixed strategy with the expected loss defined as $L_i^t(a_i)$ in Eq.~\eqref{eq:expected_loss_full} for every action $a_i\in \mathcal{A}_i$. Based on the regret analysis of the \text{Hedge} algorithm, each leader $i$ achieves no-external Stackelberg-regret learning in this process. We provide detailed descriptions of the \text{Hedge} algorithm and the proof in Appendices~\ref{appendix:hedge} and~\ref{appendix:lemma1}, respectively. 

\begin{corollary}
In the full information setting, when all leaders use Hedge as the learning algorithm, together with a reduction from no-external to no-swap regret~\citep{blum2007external,ito2020tight}, the time averaged joint strategy profile distribution $\bar{\chi}=\frac{1}{T}\sum_{t=1}^T \chi^t$ converge to an $\epsilon^T$-CSE. 
$\epsilon^T=\mathcal{O}\left(\sqrt{\frac{n\ln n}{T}}\right) \to 0$ as $T \to \infty$, which implies $\epsilon$-CSE always exists in general-sum MLSF game for any $\epsilon>0$. 
\end{corollary}

See proof in Appendix~\ref{appendix:lemma1} for the detailed proof. 

\subsection{Semi-bandit MLSF games}\label{sec:semi-bandit}
We now consider a more realistic setting with semi-bandit feedback, which is the same as the full information setting except that the leaders do not know the exact follower loss function $l_f$, but instead only receives a bandit feedback $l_i(a^t,b^t)$ and the joint action $(a_t,b_t)$ that is taken at round $t$.
In other words, at each round $t$, every leader $i$ observes the joint action $(a_t,b_t)$ that is taken, $b^t=\text{Br}(a^t)$, and can obtain the loss value $l_i(a^t,b^t)$ for that round. This is opposed to full information setting where each leader $i$ knows the \textit{expected} loss value $L_i^t(a_i)$ for any action $a_i\in\mathcal{A}_i$ defined in Eq.~\eqref{eq:expected_loss_full}. 

EXP3~\citep{auer2002nonstochastic} is a classical algorithm modified from Hedge to suit the bandit information settings, which can be used to achieve no regret learning from partial bandit feedback. Since every leader only receives information for selected actions from the environment, we use the EXP3 algorithm combined with the concentration inequality method to bound the Stackelberg-regret $R_i^S(T)$ in Eq.\eqref{eq:regret}. Formally, we have
\begin{proposition}\label{lemma:semi-bandit}
If leader $i$ uses EXP3 in semi-bandit MLSF games, it achieves no-external Stackelberg-regret learning with probability at least $1-p$, i.e.,
\[
\begin{aligned}
    R_i^S(T)\leq\mathcal{O}\left(\sqrt{Tn\ln n}+\sqrt{T\ln\frac{1}{p}}\right).
\end{aligned}
\]
\end{proposition}
Detailed descriptions of the EXP3 algirhtm and the proof of Lemma~\ref{lemma:semi-bandit} can be found in Appendices~\ref{appendix:exp3} and \ref{appendix:lemma2}, respectively.
With a reduction from no-external to no-swap regret~\citep{blum2007external,ito2020tight}, we immediately obtain that

\begin{corollary}
When all leaders use \text{EXP3} as the underlying learning algorithm in repeated general-sum semi-bandit MLSF games, the time averaged joint strategy profile distribution $\bar{\chi}=\frac{1}{T}\sum_{t=1}^T \chi^t$ converges to an approximate $\epsilon^T$-CSE with probability at least $1-p$, where
$\epsilon^T=\mathcal{O}\left(\sqrt{\frac{1}{T}n^2\ln{n}}+\sqrt{\frac{1}{T}n\ln \frac{1}{p} }\right)$.
\end{corollary}




\section{MLSF games with noisy bandit feedback}

In this section, we present our main results on a more realistic but also more challenging scenario, where we do not assume that each leader or the follower knows the form of its own loss function, but can only get \textit{noisy (stochastic)} feedback of the loss value of each round. Due to the practicality, learning equilibrium from noisy bandit feedback has been widely studied in both the game theory and online learning literature~\citep{heliou2017learning,bai2021sample}.  

In particular, the work of \citet{bai2021sample}, which is perhaps the closest to our work, considers learning equilibrium in (single) leader-follower games with noisy bandit feedback. But the analysis is based on querying batches of samples with same sizes for every action pair $(a,b)$ all at a time, as opposed to the online learning setting that we focus on. This implicitly requires a third-party authority that is able to control the sampling strategies of all the players. Our analysis does not rely on such a \textit{centralized} sampling procedure, but instead allows the players to learn from \textit{self-playing on-the-go}.

However -- online distributed learning in leader-follower games is considered harder to solve, since it cannot be restricted to evenly query every action pair. When leaders learn without a follower best response oracle, they can only use the feedback from interactions with the follower to update their strategies. But if the algorithm does not sample every leaders' joint action sufficient times, the follower cannot get enough information to learn a stable best response to the leaders. This in turn makes it hard for the leaders to learn its stablized loss since the follower's best response varies from round to round to the same leaders' joint action. In other words, leaders are not guaranteed to achieve no Stackelberg-regret learning if they seldomly choose certain actions. Therefore, it is critical to add a stable exploration to each action $a$ in the leaders' learning algorithm to avoid this issue.


\subsection{$\alpha$EXP3-UCB}
As a reminder, our goal is to design \emph{decentralized} online learning algorithms for both leaders and the follower to be able to achieve the $\epsilon$-CSE, which can be induced when $R_i^S(T)$ is sublinear in $T$ for every $i\in [m]$.

\begin{algorithm}[ht]
\caption{$\alpha$EXP3-UCB}
\label{alg:exp3ucb}
\begin{algorithmic}[1]
 \STATE {\bfseries Require:} $\eta>0$, $\beta\geq 3$
 \STATE $w_i^1=[1,\cdots,1]$, $T_k(n_a(0))=0$, $\hat{\mu}^{0}_{a,k}=0$ for any $a\in\mathcal{A}, k\in [n_f]$
 \FOR{$t=1 \ldots T$}
  \STATE Each leader $i$ sets $\widetilde{P}_i^t\!=\!(1-\alpha)P_i^t\!+\!\alpha[1/n,...,1/n]$
  \STATE Each leader $i$ draws action $a_{i}^t\sim\widetilde{P}_i^t$ 
  \STATE Follower $f$ observes $a^t$; responds with $b^t$ in Eq.\eqref{eq:follower_best_response}
  \STATE Follower $f$ observes $\xi_f^t(a^t,b^t)$ and updates $T_k(n_{a}(t))$ and $\hat{\mu}^{t}_{a, k}$ in Eqs.\eqref{eq:follower_countT}-\eqref{eq:follower_ucb_loss}.
  \STATE Each leader $i$ observes $\xi_i^t(a^t,b^t)$ and constructs the estimate $\widetilde{l}_j^t=\frac{\xi_i^t(a^t,b^t)}{\widetilde{P}_i^t(a^t)}\mathbb{I}\{a_i^t=a_{i,j}\}$ for $j \in [n]$
  \STATE Each leader $i$ updates $P_i^{t+1}$:
  $w_i^{t+1}(j) \leftarrow w_i^t(j)\cdot \exp(-\eta \widetilde{l}_j^t)$,
  $P_i^{t+1}(a_{i,j}) \leftarrow \frac{w_i^{t+1}(j)}{\sum_{j=1}^{n}w_i^{t+1}(j)}$ for $j \in [n]$
\ENDFOR
\end{algorithmic}
\end{algorithm}
Following the above intuition, we propose a new algorithm $\alpha$EXP3-UCB for learning CSE with noisy bandit feedback, as shown in Algorithm~\ref{alg:exp3ucb}. On the high-level, the algorithm is run repeatedly in $T$ rounds. At round $t$, each leader $i$ uses the $\alpha$EXP3 algorithm as the underlying learning method to sample actions (Lines 4-5) and update the strategy (Lines 8-9). For each $a\in \mathcal{A}$, the follower conducts a corresponding Upper Confidence Bound algorithm $\text{UCB}(a)$~\citep{lai1985asymptotically,auer2002finite} for the arms $k \in [n_f]$ , where $k$ is the $k$-th action of the follower, and $[n_f]=\{1,\ldots,n_f\}$ is the set of all the arms of the follower (Lines 6-7).

More specifically, the leaders' learning algorithm $\alpha$EXP3 is essentially the classical \textit{EXP3} algorithm~\citep{auer2002nonstochastic,cesa2006prediction,orabona2019modern} plus an extra explicit exploration term when selecting actions. At round $t$, each leader $i$'s joint action is selected as $a_i^t \sim \widetilde{P}_i^t$, where $\widetilde{P}_i^t=(1-\alpha)P_i^t+\alpha[1/n,...,1/n]$ is a linear combination of $P_i^t$ and a uniform probability. 
The parameter $\alpha$ can be interpreted as the minimum amount of exploration that is guaranteed. It turns out that setting an appropriate $\alpha$ is critical in balancing between sample efficiency and algorithm convergence. Lines 8-9 perform an EXP3-style update of the leaders' strategies, where $\widetilde{l}_j^t$ is an unbiased estimate of the average loss for action $j\in[n]$ using importance sampling, $P_i^{t+1}(a_{i,j})$ is the base strategy that is exponential w.r.t. the negative of $\widetilde{l}_j^t$. $\mathbb{I}\{C\}$ is an indicator function with a value of 1 when condition $C$ is met, and 0 otherwise.


For the follower, because it maintains one UCB(a) subroutine for each leader joint action $a\in\mathcal{A}$, it actually conducts $|\mathcal{A}|$ UCB algorithms. In Line 6, the follower first observes the leaders' joint action $a$, and then uses the corresponding algorithm $\text{UCB}(a)$ to obtain the response strategy: 
\begin{equation}\label{eq:follower_best_response}
    b^t = \argmin_{k\in[n_f]}A_{a^t,k}^{t-1},
\end{equation}
where
\[
A_{a,k}^{t}= \begin{cases}\hat{\mu}^{t}_{{a}, k}\!-\!\sqrt{\frac{2 \beta \ln n_{a}(t)}{T_k(n_{a}(t))}},&\!\!\!   \,\,\,\,T_k(n_{a}(t)) \!\neq\! 0 \\ -\infty, & \text { otherwise. }\end{cases}
\]
Here $n_a(t)$ is the number of times leader plays action $a$ in the first $t$ rounds, $T_k(n_a(t))$ is the number of times the follower plays its $k$-th action when leaders play action $a$ in the first $t$ rounds, and $\hat{\mu}^{t}_{{a}, k}$ is the estimated average loss of arm $k$ under leaders' joint action $a$. In Line 7, the follower then uses the observed noisy feedback $\xi_f^t(a^t,b^t)$ to update $T_k(n_{a}(t))$ and $\hat{\mu}^{t}_{a, k}$ --- for $k\in [n_f], a\in\mathcal{A}$:
\begin{align}
\label{eq:follower_countT}&T_k(n_{a}(t))=T_k(n_{a}(t-1))+\mathbb{I}\left\{a^t\!=\!a\land b^t\!=\!k\right\}, \\
\label{eq:follower_ucb_loss}&\hat{\mu}^{t}_{a, k}\!=\!\frac{1}{T_k(n_{a}(t))}\sum_{h=1}^{t}\xi_f^h(a^h,k)\mathbb{I}\left\{a^h\!=\!a\land b^h\!=\!k\right\}.
\end{align}
As $t$ increases, with the explicit exploration for the leaders, the number of times $n_a(t)$ that leaders play action $a$ also increases. Therefore, the performance of each sub-routine $\text{UCB}(a)$ is guaranteed to improve with more training samples (thus more accurate estimate of average loss $\hat{\mu}^{t}_{a, k}$) under joint leader action $a$.


\subsection{SLSF games with noisy bandit feedback} 


Before addressing the more complicated MLSF game setting, we first start with a simpler SLSF setting with a single leader. Our insight is that the MLSF games can be decomposed into several SLSF games.
Note that the result here is a non-trivial improvement over the state-of-the-art result~\citep{bai2021sample} which focuses on SLSF games with noisy bandit feedback. We prove that we can learn a Starkelberg equilibrium in a SLSF bandit game with noisy bandit feedback in a distributed online manner, where the players can reach the equilibrium via self-playing as opposed to being controlled by a centralized learner. 

Since we only consider two players in this subsection, we simplify the notation here. The single leader plays an action $a \in \mathcal{A}$, and observes it own noisy loss $\xi(a, b) \in [0, 1]$, $\mathbb{E}[\xi(a, b)]=l(a, b)$. The follower observes a loss $\xi_f(a, b) \in [0, 1]$, $\mathbb{E}[\xi_f(a, b)]=l_f(a, b)$. We prove that:



\begin{theorem}\label{thm1}
Applying $\alpha$EXP3-UCB to a SLSF game with noisy bandit feedback, with $\alpha = \mathcal{O}\left({n^{\frac{2}{3}}(\ln n)^{\frac{1}{3}}T^{-\frac{1}{3}}}\right)$, each leader achieves no-external Stackelberg-regret with probability at least $1-p$, i.e.,
\[
\begin{aligned}
&R^S(T)=\mathbb{E}\left[\sum_{t=1}^T l(a^t,\text{Br}(a^t))-l(a_\star,\text{Br}(a_\star))\right]\\
&\leq \widetilde{\mathcal{O}}\left(n^{\frac{1}{3}}n_f\frac{\beta}{\varepsilon^2}T^{\frac{1}{3}}(\ln\frac{1}{p})^2+n^\frac{2}{3}(\ln n)^\frac{1}{3}T^\frac{2}{3}+n_f n \frac{\beta}{\beta-2}\right),
\end{aligned}
\]
where $a_\star = \argmin_{a\in \mathcal{A}}\sum_{t=1}^T l(a,\text{Br}(a))$, and the $\widetilde{\mathcal{O}}(\cdot)$ notation hides factors that are polynomial in $\ln T$.
\end{theorem} 
\textit{Proof sketch.}
We present the high-level idea of our proof here and refer to Appendix~\ref{appendix:theroem1} for the full proof.
We decompose the noisy Stackelberg-regret of the leader into the following three terms
\[
\begin{aligned}
\widetilde{R^S}(T)=&\mathop{\mathbb{E}}\left[\sum_{t=1}^T \xi^t(a^t,\text{Br}(a^t))-\xi^t(a_\star,\text{Br}(a_\star))\right] \\
=&\mathop{\mathbb{E}}\left[\sum_{t=1}^T \xi^t(a^t,\text{Br}(a^t))-\xi^t(a^t,b^t(a^t))\right] (\text{I})\\
+&\mathop{\mathbb{E}}\left[\sum_{t=1}^T \xi^t(a^t,b^t(a^t))-\xi^t(a_\star,b^t(a_\star))\right] (\text{II})\\
+&\mathop{\mathbb{E}}\left[\sum_{t=1}^T \xi^t(a_\star,b^t(a_\star))-\xi^t(a_\star,\text{Br}(a_\star))\right] (\text{III})
\end{aligned}
\]
and bound each term separately. $b^t(a)$ represents the the follower's response to $a$ at round $t$, which is determined by the follower's subroutine UCB($a$) at round $t$, and is not necessarily equal to the best response $\text{Br}(a)$. 

First, since $\xi(a,b)\in [0,1]$ for any $(a,b)\in \mathcal{A}\times\mathcal{B}$, the first term (\text{I}) and the third term (\text{III}) can be respectively bounded with the following two inequalities:
\begin{small}
\[
\sum_{t=1}^T \left|\xi^t(a^t,\text{Br}(a^t))-\xi^t(a^t,b^t(a^t))\right|\leq \sum_{t=1}^T \mathbb{I}\{b^t(a^t)\neq \text{Br}(a^t)\}
\]
\end{small}
\begin{small}
\[
\sum_{t=1}^T \left|\xi^t(a^t,\text{Br}(a^t))-\xi^t(a_\star,b^t(a_\star))\right|\leq \sum_{t=1}^T \mathbb{I}\{b^t(a_\star)\neq \text{Br}(a_\star)\}
\]
\end{small}

To further bound these two terms to be sublinear in $T$, the idea is to guarantee that the number of times that a suboptimal arm is played is sublinear in $T$. Because the underlying algorithm of the follower is UCB, the sublinearity will be satisfied after the UCB subroutines explore sufficient rounds.
Therefore, a critical step is to make sure that every $a\in \mathcal{A}$ will be played sufficient rounds, or more specifically, $a$ needs to be played at least once in a time interval sublinear in $T$ (e.g., $\mathcal{O}(T^{\frac{1}{3}})$) with high probability. This is satisfied by the extra $\alpha$-explicit exploration in the leaders' part of the algorithm. By combining the above steps with concentration inequality, term $(\text{I})$ and term $(\text{III})$ can be bounded as follows --- for a sufficient small $p$,
\[
(\text{I})\leq\mathcal{O}\left(n_f n\frac{8\beta\ln T}{\varepsilon^2}+\sqrt{T\ln\frac{1}{p}}\right),
\]

\[
(\text{III})\leq \mathcal{O}\left(\frac{1}{\alpha}n_f n(\ln\frac{1}{p})^2 \frac{8\beta\ln T}{\varepsilon^2}\right).
\]

Second, in term (II), since the follower response is consistently in the same form, it can be treated as part of the environment. Hence term (II) is essentially the regret of the $\alpha$EXP3 algorithm. Based on this observation, we bound term $(\text{II})$
through a regret analysis on the $\alpha$EXP3 algorithm, together with an adaptation that uses an additional analysis on the losses for extra exploration:
\[
(\text{II})\leq \mathcal{O}\left(\frac{\ln n}{\eta}+\frac{\eta n^2 T}{\alpha}+2\alpha T\right).
\]

Then, we set an appropriate explicit exploration parameter $\alpha$ and learning rate $\eta$ as follows so that it bounds each term to be sublinear in $T$:
\[
\alpha = \mathcal{O}\left({n^{\frac{2}{3}}(\ln n)^{\frac{1}{3}}T^{-\frac{1}{3}}}\right), \eta = \mathcal{O}\left({n^{-\frac{2}{3}}(\ln n)^{\frac{2}{3}}T^{-\frac{2}{3}}}\right).
\]

Last, after we bound the noisy Stackelberg-regret, we bound true Stackelberg-regret for the leader using concentration inequality:
\[
R^S(T)\leq \widetilde{R^S}(T)+\mathcal{O}\left(\sqrt{T\ln\frac{1}{p}}\right).
\]

Following Theorem~\ref{thm1}, we immdediately have:


\begin{theorem}
Using $\alpha$EXP3-UCB to learn in a SLSF game with noisy bandit feedback, set $\alpha = \mathcal{O}\left({n^{\frac{2}{3}}(\ln n)^{\frac{1}{3}}T^{-\frac{1}{3}}}\right)$, with probability at least $1-p$, the joint empirical strategy profile $\Bar{P}=\frac{1}{T}\sum_{t=1}^T \widetilde{P}^t$ is an approximate Starckelberg equilibrium,
\[
\mathbb{E}_{a \sim \Bar{P}}\left[l(a,\text{Br}(a)) \right] \leq l(a_\star,\text{Br}(a_\star))  + \epsilon^T,
\]

where $\epsilon^T=\widetilde{\mathcal{O}}\left(n^{\frac{1}{3}}n_f\frac{\beta}{\varepsilon^2}T^{-\frac{2}{3}}(\ln\frac{1}{p})^2+ n^\frac{2}{3}(\ln n)^\frac{1}{3}T^{-\frac{1}{3}}\right)$.
\end{theorem}



Note that the Stackelberg equilibrium is a special case of CSE in MLSF games when there is only one leader.

\subsection{MLSF games with noisy bandit feedback}

We now study the more complicated scenario of MLSF games with noisy bandit feedback. Compared to the above SLSF bandit games, MLSF bandit games are generally harder because the size of joint action space $|\mathcal{A}|=|\mathcal{A}_1| \times |\mathcal{A}_2| \times \cdots \times |\mathcal{A}_m|=n^m$ increases exponentially w.r.t. $m$. 
In this case the explicit exploration parameter $\alpha$ needs to be big enough to guarantee that the expected number of times that each joint leader action $a\in\mathcal{A}$ being played (i.e. $\mathbb{E}\left[\sum_{t=1}^T \mathbb{I}\{a^t=a\}\right]=\left(\frac{\alpha}{n}\right)^m T$) is at least sublinear in $T$.


For the special case of two leaders ($m=2$), when all players use $\alpha$EXP3-UCB, we can still get a sub-linear no-external Stackelberg regret $R_i^S(T)\leq\widetilde{\mathcal{O}}\left(n_f n^{\frac{2}{3}}\frac{\beta}{\varepsilon^2}T^\frac{2}{3}(\ln \frac{1}{p})^2+n_f n \frac{\beta}{\beta-2}\right)$ for $i=1,2$ with probability at least $1-p$ by setting the exploration parameter $\alpha = \mathcal{O}\left({n^{\frac{2}{3}}(\ln n)^{\frac{1}{3}}T^{-\frac{1}{3}}}\right)$ 
(see proof in Appendix~\ref{appendix:theorem3}). However, if we still use Algorithm~\ref{alg:exp3ucb} for cases when $m>2$, we need to set a much larger $\alpha$ to encourage a more agressive exploration. This results in an extremely slow convergence. Formally:
\begin{theorem}
Using $\alpha$\text{EXP3-UCB} for MLSF games with noisy bandit feedback, with $\alpha= \mathcal{O}\left(nT^{-\frac{1}{m+1}}\right)$, $T\geq\mathcal{O}(n^{m+1})$, define $L_i^t(a_i)$ in Eq.\eqref{eq:expected_loss_full}, it achieves no-external regret learning with probability at least $1-p$
\[
\begin{aligned}
&R_i^S(T)=\mathbb{E}\left[\sum_{t=1}^T L_i^t(a_i^t)-L_i^t(a_{i,\star})\right]\\
&\leq \widetilde{\mathcal{O}}\left(\left(n+n_f \frac{\beta}{\varepsilon^2}(\ln \frac{1}{p})^2\right)T^{\frac{m}{m+1}}+n_f n \frac{\beta}{\beta-2}\right).
\end{aligned}
\]
\end{theorem}
See proof in Appendix~\ref{appendix:theorem3}. Because $\alpha\leq 1$ and $\alpha=\mathcal{O}\left(nT^{-\frac{1}{m+1}}\right)$, we require that $T\geq \mathcal{O}(n^{m+1})$. Intuitively, the big exploration parameter $\alpha$ incurs very low sample efficiency and more regret because it suffers from losses by unnecessarily exploring many ``bad'' actions. Although the term $n^m$ can not be avoided in general (because we need to enumerate every leaders-follower action pair $(a,b)$ sufficient times to learn a reasonable estimated loss function), we can still find a more efficient algorithm to improve the regret w.r.t. $T$. 

\subsubsection{A sample-efficient two-stage learning algorithm} 
To overcome the above issue, our intuition is that the players should initially use more aggressive exploration to obtain exact best response with high probability, and then reduce exploration for the sake of algorithm convergence. Based on this intuition, we propose a two-stage learning algorithm (see Algorithm~\ref{alg: ucbe}) which is provably more efficient than Algorithm~\ref{alg:exp3ucb}. 

In the first stage (when round $t\leq t_0$; Lines 2-7), the leaders perform a pure random exploration without updating their strategies (Lines 3\&5), while the follower uses the highly explorative type algorithm Upper Confidence Bound Exploration (UCB-E)~\citep{audibert2010best} to select its best response (Line 4) and update its strategy (Line 6). In Line 4, the best response is chosen as
\begin{equation}\label{eq:follower_response}
    b^t = \argmin_{k\in[n_f]}B_{a^t,k}^{t-1},
\end{equation}
where 
\[
B_{a,k}^{t}\!=\! \begin{cases}\hat{\mu}^{t}_{a, k}\!-\!\sqrt{\frac{e}{T_k(n_a(t))}}, &   T_k(n_{a}(t)) \!\neq\! 0 \\ -\infty, & \text { otherwise. }\end{cases}
\]
Here $e$ is a parameter that specifies the extent of exploration. Because of the way that an arm is selected, UCB-E is a highly explorative type algorithm designed for solving the best arm identification problem in multi-armed bandit games. 
In Line 6, the strategy update is following the same UCB-style as in Algorithm~\ref{alg:exp3ucb}, where $T_k(n_a(t))$ and $\hat{\mu}^{t}_{a, k}$ are respectively updated by Eq.\eqref{eq:follower_countT} and Eq.\eqref{eq:follower_ucb_loss}.

As the end of the first stage (Line 9), the follower learns the best response predictors $\left\{\widehat{\text{Br}}(a),a\in \mathcal{A}\right\}$ that has the minimal estimated average loss up to round $t_0$:
\begin{equation}\label{bestr}
  \left\{\widehat{\text{Br}}(a),a\in \mathcal{A}\right\}:\widehat{\text{Br}}(a)=\argmin_{k\in [n_f]}\hat{\mu}^{t_0}_{a, k}.
\end{equation}
Since leaders conduct pure and explicit exploration in the first stage, the follower's best response predictor is found by each sub-routine UCB-E($a$) with a high probability due to the sufficient exploration.

In the second stage (when $t>t_0$; Lines 9-13), the leaders then switch to EXP3 to update their strategies, while the follower \textit{commits} to the strategy learned from the first-stage exploration and stops updating it. Because the follower stops updating its strategy in this stage, it actually reduces to the semi-bandit setting in Section~\ref{sec:semi-bandit}.

Overall, for the leaders' algorithm, it is equivalent to the $\alpha^t$EXP3 algorithm by setting the explicit exploration parameter $\alpha_t$ to be $\alpha_t=\begin{cases}1, &   t \leq t_0 \\ 0, & t > t_0\end{cases}$, and not performing strategy update in the first stage. For the follower's algorithm, it is essentially a \textit{learn-to-commit} procedure.

\begin{algorithm}[ht]
\caption{Two-stage bandit algorithm}
\label{alg: ucbe}
\begin{algorithmic}[1]
\STATE $T_k(n_a(0))=0$, $\hat{\mu}^{0}_{a,k}=0$ for any $a\in\mathcal{A}, k\in [n_f]$
 \FOR{$t=1,2,\cdots,t_0$}
  \STATE Each leader $i$ selects $a_{i}^t\sim [1/n,\ldots,1/n]$ 
  \STATE Follower $f$ observes $a^t$; responds with $b^t$ in Eq.\eqref{eq:follower_response}
  \STATE Each leader $i$ receives $\xi_i^t(a^t,b^t)$; no strategy update 
  \STATE Follower $f$ receives $\xi_f^t(a^t,b^t)$ and updates $T_k(n_a(t))$ and $\hat{\mu}^{t}_{a, k}$ by Eq.~\eqref{eq:follower_countT} and Eq.~\eqref{eq:follower_ucb_loss}
 \ENDFOR
 \STATE Follower $f$ learns best response predictors
 $\left\{\widehat{\text{Br}}(a),a\in \mathcal{A}\right\}$ in Eq.\eqref{bestr} 
 
\FOR{$t=t_0+1,t_0+2,\cdots,T$}
  \STATE Each leader $i$ selects $a_{i}^t$ with EXP3
  \STATE Follower $f$ observes $a^t$ and selects $b^t=\widehat{\text{Br}}(a^t)$
  \STATE Each leader $i$ receives $\xi_i(a^t,b^t)$ and updates strategy with EXP3
\ENDFOR
\end{algorithmic}
\end{algorithm}


\subsubsection{Two stage learning results}

Before presenting the our final results, we first have the following Lemma as a prerequisite:
\begin{lemma}\label{lemma3} \citep{audibert2010best} 
If UCB-E is run with parameter $0<e\leq \frac{25}{36}\frac{T_a - n_f}{H_a}$, let $T_a$ be the number of times leaders choose $a$ in the first $t_0$ rounds, $H_a=\sum_{k=1}^{n_f} \frac{1}{\Delta_{ak}^2}$, then it satisfies
\[
\mathbb{P}\left(\widehat{\text{Br}}(a)\neq \text{Br}(a)\right)\leq 2T_a n_f \exp\left(-\frac{2e}{25}\right).
\]
In particular, when we set $e=\frac{25}{36}\frac{T_a - n_f}{H_a}$, we have $\mathbb{P}\left(\widehat{\text{Br}}(a)\neq \text{Br}(a) \right)\leq 2T_a n_f \exp\left(-\frac{T_a - n_f}{18H_a}\right)$.
\end{lemma}
Lemma~\ref{lemma3} guarantees that the follower learns a best response predictor with a high probability for any $a\in\mathcal{A}$ through UCB-E in the first stage of the algorithm. Building on top of that, we have

\begin{theorem}
For a MLSF game with noisy bandit feedback, if every leader $i$ uses Algorithm~\ref{alg: ucbe}, let $q\geq 18H_a\left(\ln\frac{2q n_f}{p}+m\ln n\right)+n_f$, $e=\frac{25}{36}\frac{q - n_f}{H_a}$, $t_0=\mathcal{O}\left(n^m q\right)$, and define $R_i^S(T)$ as in Eq.~\eqref{eq:regret},
with probability at least $1-2p$, we have
\[
\begin{aligned}
&R_i^S(T)\leq \mathcal{O}\left(t_0+\sqrt{Tn\ln n}+\sqrt{T\ln \frac{1}{p}}\right)
\end{aligned}
\]
and use a reduction from no-external to no-swap regret~\citep{blum2007external,ito2020tight}, we get an $\epsilon^T$-CSE for leaders in MLSF bandit game, where $\epsilon^T=\mathcal{O}\left(\frac{t_0}{T}+\sqrt{\frac{1}{T}n^2\ln n}+\sqrt{\frac{1}{T}n\ln \frac{1}{p}}\right)$. And the follower learns best response predictor with a high probability for any $a\in\mathcal{A}$
\[
\mathbb{P}\left(\widehat{\text{Br}}(a)\neq \text{Br}(a) \right)\leq \frac{p}{n^m}.
\]
\end{theorem}

We refer to the full proof in Appendix~\ref{appendix:theorem4}. 
It is worth mentioning that the term $\mathcal{O}\left(n^m\right)$ is inevitable, since in the bandit noisy feedback setting we need to go through every action in the action space $\mathcal{A}$ ($|\mathcal{A}|=n^m$) sufficient times to find the follower best response $\text{Br}(a)$ for every $a$. 

Based on Lemma~\ref{lemma3}, the follower can get the exact best response predictors $\widehat{\text{Br}}(a)=\text{Br}(a)$ for every $a\in \mathcal{A}$ with probability at least $1-p$ using the union bound. After the follower commits its identification, the follower will play a fixed best response action for each $a\in\mathcal{A}$ in the second stage. With the reliable best response predictors learned by the follower, i.e., with a follower best response oracle that holds with high probability, the second-stage game is reduced to a semi-bandit MLSF game for leaders in Section~\ref{sec:semi-bandit} and therefore the proof of Lemma~\ref{lemma:semi-bandit} can be re-used here with a slight adaptation.


\section{Conclusion}
This paper is the first to take a learning perspective of general-sum multi-leader-single-follower games. We first propose a new viable equilibrium concept called correlated Stackelberg equilibrium and prove its existence in full information and semi-bandit settings. We then study the more challenging setting where leaders and the follower can only obtain noisy bandit feedback, and prove convergence results of our proposed learning algorithms. Our work opens up many potential future directions at the intersection of learning and MLSF games. For example, it is interesting to see: 1) Can we relax the type of learning algorithms for the players? 2) Can we find computationally more efficient algorithms to reach the CSE? 3) Can our results be generalized to the SLMF and MLMF settings? 

\section*{Societal impact}
General-sum MLSF games have broad applications in many real-world problems, such as security games, wild-life conservation, deregulated electricity markets, and industrial eco-parks. Our work studies the learning aspects of this set of problems, and takes a key step by allowing decentralized learning among the players. Our proposed algorithms, together with the fundamental theoretical analysis, lay the foundation of applying realistic learning algorithms towards the set of practical problems, and therefore can create huge societal impact to the above domains. 


\bibliography{mlsf}
\bibliographystyle{apalike}

\onecolumn
\appendix
\section{The Hedge algorithm}\label{appendix:hedge}
\begin{algorithm}[ht]
\caption{Hedge for Leader $i$}
\label{alg:hedge}
\begin{algorithmic}
\STATE {\bfseries Require:} $\eta, T \in N$
 \STATE $P_i^1=[1/n,\cdots,1/n]$, $w_i^1=[1,\cdots,1]$
 \FOR{$t=1,\cdots,T$}
  \STATE Leader $i$ observes $P_{-i}^t$ and calculates $L_i^t(a_{i,j})$, 
  for each $j \in [n_i]$, $L_i^t(a_{i,j})$ is defined by Eq.\eqref{eq:expected_loss_full}
  \STATE Leader $i$ updates $P_i^{t+1}$:
  $w_i^{t+1}(j) \leftarrow w_i^t(j)\cdot \exp(-\eta L_i^t(a_{i,j}))$,
  $P_i^{t+1}(j) \leftarrow \frac{w_i^{t+1}(j)}{\sum_{j=1}^{n}w_i^{t+1}(j)}$ for $j \in [n]$
 \ENDFOR
\end{algorithmic}
\end{algorithm}

\section{Proof of Proposition 1}\label{appendix:lemma1}
\begin{lemma}
Let $l_t^2$ denote the n-dimensional vector of square losses, i.e., $l_t^2(i)=(l_t(i))^2$, let $\eta>0$, and assume all losses to be non-negative. The Hedge alhorithm satisfies for any expert $i^\star\in [n]$
\[
\sum_{t=1}^T \textbf{x}_t^\top l_t -\sum_{t=1}^T l_t(i^\star)
\leq \frac{\ln n}{\eta}+\eta\sum_{t=1}^T \textbf{x}_t^\top l_t^2.
\]
\end{lemma}
\begin{proof}
See Theorem 1.5 of \citet{hazan2019introduction} for a detailed proof.
\end{proof}

Based on the regret analysis of algorithm Hedge~\citep{hazan2019introduction}, noticed that $L_i^t(a_{i,j})\in [0,1] $ for all $j\in [n]$, we can bound the Stackelberg-regret for leader $i$ 
\[
\begin{aligned}
\sum_{t=1}^T\mathbb{E}_{a_{i}^t\sim P_{i}^t}[L_i^t(a_{i}^t)-L_i^t(a_{i,\star})]&\leq \frac{\ln n}{\eta}+ \eta\sum_{t=1}^T\sum_{j=1}^n P_i^t(a_{i,j})\cdot(L_i^t(a_{i,j}))^2\\
&\leq \frac{\ln n}{\eta}+\eta\sum_{t=1}^T\sum_{j=1}^n P_i^t(a_{i,j}) & \left(P_i^t \;\text{is a distribution, so\;} \sum_{j=1}^n P_i^t(a_{i,j})=1\right)\\
&=\eta T+\frac{\ln n}{\eta}.
\end{aligned}
\]
We choose $\eta =\mathcal{O}\left( \frac{\ln n_i}{T}\right)$, then we have

\[
\begin{aligned}
  R_{i}^S\left(T\right)
  &=\sum_{t=1}^T\mathbb{E}_{a_{i}^t\sim P_{i}^t}[L_i^t(a_{i}^t)-L_i^t(a_{i,\star})]  \\
  &\leq \eta T+\frac{\ln n}{\eta} \leq \mathcal{O}\left(\sqrt{T\ln n }\right).
\end{aligned}
\]

Using reduction from no-external to no-swap regret, which is Theorem 2 of~\citet{ito2020tight}, for any swap function $s: \mathcal{A}_i\to \mathcal{A}_i$, we have
\[
R_{i, \text{swap}}^S\left(T\right)=\sum_{t=1}^T\mathbb{E}_{a_{i}^t\sim P_{i}^t}[L_i^t(a_{i}^t)-L_i^t(s(a_{i}^t))]\leq \mathcal{O}\left(\sqrt{nT\ln n }\right).
\]

Let $\epsilon^T=R_{i,\text{swap}}^S\left(T\right)/T$, for any $i\in [n]$, we have,
\[
\frac{1}{T}\sum_{t=1}^T\mathbb{E}_{a\sim \chi^t}\left[l_i(a^t,\text{Br}(a^t))\right]\leq
\frac{1}{T}\sum_{t=1}^T\mathbb{E}_{a\sim \chi^t}\left[l_i(s(a_{i}),Br(s(a_{i}),a_{-i}))\right]+\epsilon^T.
\]

Since expectations are linear, we 
rewrite the inequality using the time averaged joint action profile distribution $\bar{\chi}=\frac{1}{T}\sum_{t=1}^T \chi^t$ as follows

\[
\mathbb{E}_{a\sim \bar{\chi}}\left[l_i(a,Br(a))\right]\leq \mathbb{E}_{a\sim \bar{\chi}}\left[l_i(s(a_{i}),Br(s(a_{i}),a_{-i}))\right]+\epsilon^T.
\]
We can use the following process to sample a joint action from distribution $\Bar{\chi}$: we first sample $t$ uniformly from $[T]$, we then sample a joint action $a$ from the distribution $\chi^t$.
$\epsilon^T=\mathcal{O}\left(\sqrt{\frac{n\ln n}{T}}\right) \to 0$ as $T \to \infty$, which implies that $\epsilon$-CSE exists in general-sum MLSF game for any $\epsilon>0$.

\section{The Exp3 algorithm}\label{appendix:exp3}

\begin{algorithm}[ht]
\caption{EXP3 for Leader $i$ }
\label{alg:exp3}
\begin{algorithmic}
\STATE {\bfseries Require:} $\eta, T \in N$
\STATE $P_i^1=[1/n,\cdots,1/n]$ 
\FOR{$t=1,\cdots,T$}
  \STATE Leader $i$ draws $a_i^t$ according to $P_i^t$ and selects $a_i^t$
  \STATE Follower observes $a^t$ and plays $b^t=\text{Br}(a^t)$ 
  \STATE Leader $i$ observes $l_i(a^t,b^t)$ and constructs the estimate $\widetilde{l}_{i,j}^t=\frac{l_i(a^t,b^t)}{P_i^t(a^t)}\mathbb{I}\{a_i^t=a_{i,j}\}$ for $j \in [n]$
  \STATE Leader $i$ updates $P_i^{t+1}$: $P_i^{t+1}(a_{i,j}) \propto P_i^t(a_{i,j})\cdot \exp(-\eta \widetilde{l}_{i,j}^t)$
 \ENDFOR
\end{algorithmic}
\end{algorithm}

\section{Proof of Proposition 2}\label{appendix:lemma2}
In this section, for simplicity, we let $\widetilde{L}_i^t(a_{i})=l_i(a_i,a_{-i}^t,\text{Br}(a_i,a_{-i}^t))$, $L_i^t(a_{i})$ be defined by Eq.\eqref{eq:expected_loss_full}, $a_i\in\mathcal{A}_i$.

We decompose the Stackelberg-regret for leader $i$ as follows,
\[
\begin{aligned}
R_{i}^S\left(T\right)
&=\sum_{t=1}^T\mathbb{E}_{a_{i}^t\sim P_{i}^t}[L_i^t(a_{i}^t)-L_i^t(a_{i,\star})]\\
&=\underbrace{\sum_{t=1}^T\mathbb{E}_{a_{i}^t\sim P_{i}^t}[L_i^t(a_{i}^t)-\widetilde{L}_i^t(a_{i}^t)]}_{\text{Term I}}+
\underbrace{\sum_{t=1}^T\mathbb{E}_{a_{i}^t\sim P_{i}^t}[\widetilde{L}_i^t(a_{i}^t)-\widetilde{L}_i^t(a_{i,\star})]}_{\text{Term II}}+
\underbrace{\sum_{t=1}^T\widetilde{L}_i^t(a_{i,\star})-L_i^t(a_{i,\star})}_{\text{Term III}}.
\end{aligned}
\]
\text{Term I} and \text{Term III} are caused by the randomness of leaders choosing their actions from their mixed strategy distribution profile at each round. \text{Term II} is the regret caused by losses generated by the other leaders' selected actions at each round. First, we bound \text{Term II} based on the regret bound of the EXP3 algorithm. Then we bound \text{Term I} and \text{Term III} by concentration inequality methods.

Based on the regret analysis of the EXP3 algorithm, which can be found in Theorem 10.2 of \citet{orabona2019modern}, by setting $\eta =\mathcal{O}\left( \frac{\ln n_i}{T}\right)$, we can bound \text{Term II} as
\[
\begin{aligned}
\text{Term II}=\sum_{t=1}^T\mathbb{E}_{a_{i}^t\sim P_{i}^t}[\widetilde{L}_i^t(a_{i}^t)-\widetilde{L}_i^t(a_{i,\star})]&\leq \mathcal{O}\left(\sqrt{nT\ln n}\right).
\end{aligned}
\]

For any $a_i\in \mathcal{A}_i$, $t\in [T]$

\[
\mathbb{E}_{a_{-i}^t\sim P_{-i}^t}\left[\widetilde{L}_i^t(a_{i})\right]=L_i^t(a_{i}), -1\leq\widetilde{L}_i^t(a_{i})-L_i^t(a_{i})\leq 1.
\]
Using concentration inequality, we have
\[
\mathbb{P}\left[\left|\sum_{t=1}^T\widetilde{L}_i^t(a_{i})-L_i^t(a_{i})\right|> \varepsilon\right]\leq 2\exp\left(\frac{-2\varepsilon^2}{4T}\right)=2\exp\left(\frac{-\varepsilon^2}{2T}\right).
\]

Then, for any $\delta>0$, the following inequality holds with probability at least $1-\delta$ for any $a_i\in \mathcal{A}_i$:

\[
\left|\sum_{t=1}^T\widetilde{L}_i^t(a_{i})-L_i^t(a_{i})\right|\leq \sqrt{2T\ln{\frac{2}{\delta}}}.
\]

So we can bound \text{Term III} with probability at least $1-\delta$
\[
\text{Term III}\leq\left|\sum_{t=1}^T\widetilde{L}_i^t(a_{i,\star})-L_i^t(a_{i,\star})\right|\leq \sqrt{2T\ln{\frac{2}{\delta}}}.
\]

Similarly, using concentration inequality,  the following inequality holds with probability at least $1-\delta$

\[
\left|\sum_{t=1}^T L_i^t(a_{i}^t)-\widetilde{L}_i^t(a_{i}^t)-\mathbb{E}_{a_{i}^t\sim P_{i}^t}[L_i^t(a_{i}^t)-\widetilde{L}_i^t(a_{i}^t)]\right|\leq 2\sqrt{2T\ln{\frac{2}{\delta}}}.
\]
So we can bound \text{Term I} with probability at least $1-2\delta$

\[
\begin{aligned}
\text{Term I}&=\sum_{t=1}^T\mathbb{E}_{a_{i}^t\sim P_{i}^t}[L_i^t(a_{i}^t)-\widetilde{L}_i^t(a_{i}^t)]\\
&=\sum_{t=1}^T\mathbb{E}_{a_{i}^t\sim P_{i}^t}[L_i^t(a_{i}^t)-\widetilde{L}_i^t(a_{i}^t)]-\sum_{t=1}^T \left[L_i^t(a_{i}^t)-\widetilde{L}_i^t(a_{i}^t)\right]+\sum_{t=1}^TL_i^t(a_{i}^t)-\widetilde{L}_i^t(a_{i}^t)\\
&\leq 2\sqrt{2T\ln{\frac{2}{\delta}}}+\sqrt{2T\ln{\frac{2}{\delta}}}=3\sqrt{2T\ln{\frac{2}{\delta}}}.
\end{aligned}
\]

With probability at least $1-p$, $p=3\delta$, we have the following inequality

\[
\begin{aligned}
R_{i}^S\left(T\right)
  &=\sum_{t=1}^T\mathbb{E}_{a_{i}^t\sim P_{i}^t}[L_i^t(a_{i}^t)-L_i^t(a_{i,\star})]\\
  &=\sum_{t=1}^T\mathbb{E}_{a_{i}^t\sim P_{i}^t}[L_i^t(a_{i}^t)-\widetilde{L}_i^t(a_{i}^t)]+
  \sum_{t=1}^T\mathbb{E}_{a_{i}^t\sim P_{i}^t}[\widetilde{L}_i^t(a_{i}^t)-\widetilde{L}_i^t(a_{i,\star})]+
  \sum_{t=1}^T\widetilde{L}_i^t(a_{i,\star})-L_i^t(a_{i,\star})\\
  &\leq \mathcal{O}\left(\sqrt{T\ln{\frac{1}{p}}}\right)+\mathcal{O}\left(\sqrt{nT\ln n}\right)+\mathcal{O}\left(\sqrt{T\ln{\frac{1}{p}}}\right)\\
  &=\mathcal{O}\left(\sqrt{nT\ln n}+\sqrt{T\ln{\frac{1}{p}}}\right).
\end{aligned}
\]

Using reduction from no-external to no-swap regret, which is Theorem 2 of~\citet{ito2020tight}, we have
\[
R_{i, \text{swap}}^S\left(T\right)\leq \mathcal{O}\left(\sqrt{n^2 T\ln n} +\sqrt{nT\ln{\frac{1}{p}}}\right).
\]
Let $\epsilon^T=R_{i,\text{swap}}^S\left(T\right)/T$, for any $i\in [n]$, we have,
\[
\frac{1}{T}\sum_{t=1}^T\mathbb{E}_{a\sim \chi^t}\left[l_i(a^t,\text{Br}(a^t))\right]\leq
\frac{1}{T}\sum_{t=1}^T\mathbb{E}_{a\sim \chi^t}\left[l_i(s(a_{i}),Br(s(a_{i}),a_{-i}))\right]+\epsilon^T.
\]

Since expectations are linear, we can
rewrite the inequality using the time averaged joint action profile distribution $\bar{\chi}=\frac{1}{T}\sum_{t=1}^T \chi^t$ as follows

\[
\mathbb{E}_{a\sim \bar{\chi}}\left[l_i(a,Br(a))\right]\leq \mathbb{E}_{a\sim \bar{\chi}}\left[l_i(s(a_{i}),Br(s(a_{i}),a_{-i}))\right]+\epsilon^T.
\]
We can use the following process to sample a joint action from distribution $\Bar{\chi}$: first, we sample $t$ uniformly from $[T]$. Then, we sample a joint action $a$ from the distribution $\bar{\chi}$. So $\Bar{\chi}$ is an $\epsilon^T$-CSE.

\section{Proof of Theorem 1}\label{appendix:theroem1}
We decompose the noisy Stackelberg-regret of the leader into the following three terms
\[
\begin{aligned}
\widetilde{R^S}(T)=&\mathbb{E}\left[\sum_{t=1}^T \xi^t(a^t,\text{Br}(a^t))-\xi^t(a_\star,\text{Br}(a_\star))\right] \\
=&\underbrace{\mathbb{E}\left[\sum_{t=1}^T \xi^t(a^t,\text{Br}(a^t))-\xi^t(a^t,b^t(a^t))\right]}_{\text{Term I}}\\
+&\underbrace{\mathbb{E}\left[\sum_{t=1}^T \xi^t(a^t,b^t(a^t))-\xi^t(a_\star,b^t(a_\star))\right]}_{\text{Term II}} 
+\underbrace{\mathbb{E}\left[\sum_{t=1}^T \xi^t(a_\star,b^t(a_\star))-\xi^t(a_\star,\text{Br}(a_\star))\right]}_{\text{Term III}}.
\end{aligned}
\]

and bound each term separately. $b^t(a)$ represents the the follower's response to $a$ at round $t$, which is determined by the follower's subroutine UCB($a$) at round $t$, and is not necessarily equal to the best response $\text{Br}(a)$. 

Intuitively, term (I) means the cumulative gap between the noisy losses of the leaders' action $a^t$ when the follower respectively uses \textbf{any response $b^t(a^t)$} or the \textbf{best response $\text{Br}(a^t)$}. Term (II) represents the cumulative gap between the noisy losses of the leaders' \textbf{any action $a^t$} and \textbf{optimal action $a_\star$}, when the follower plays response $b^t(a)$. Term (III) is the cumulative gap between noisy losses of the leaders' optimal action $a_{\star}$ when the follower uses \textbf{any response $b^t(a_{\star})$} or \textbf{best response $\text{Br}(a_{\star})$}.

\subsection{Bound $T_k(n_a(T))$ for suboptimal response $k\neq \text{Br}(a)$}
Since $\xi_{f}(a,b)\in [0,1]$, $\xi_{f}(a,b)-\mathbb{E}\left[\xi_{f}(a,b)\right]$ is 1-subgaussian variable. This is based on the fact that if a random variable $X$ has a mean of zero and $X\in[a,b]$ almost surely, then $X$ is $(b-a)/2$-subgaussian.

\begin{lemma}
Assume that the losses of the follower's actions minus their expectations are 1-subgaussian and $\beta\geq 3$. Then, using UCB algorithm guarantees that for any $a\in \mathcal{A}$, if $k \neq \text{Br}(a)$, then
\[
\mathbb{E}[T_k(n_a(T))]\leq \frac{8\beta\log(n_a(T))}{\Delta_{ak}^2}+\frac{\beta}{\beta-2}.
\]

And if at time $t$,  $T_k(n_a(t-1))>\frac{8\beta\log(n_a(T))}{\Delta_{ak}^2}$, then for $k\neq \text{Br}(a)$
\[
\mathbb{I}\{a^t=a\}\mathbb{P}\left(\mathbb{I}\{ b^t=k\}\right)\leq 2(n_a(t)-1)n_a(t)^{-\beta}.
\]
\end{lemma}

\begin{proof}
See Theorem 10.14 of \citet{orabona2019modern} for a detailed proof.
\end{proof}


Assume that action $a$ is chosen at round $t_1,t_2,\cdots,t_{n_a(T)}$.
For any suboptimal $k$ for leader's action $a$, if $T_k(n_a(T-1))\leq \frac{8\beta\log(n_a(T))}{\Delta_{ak}^2}$, then we have $T_k(n_a(T))\leq \frac{8\beta\log(T)}{\Delta_{ak}^2}+1$. We next assume that $T_k(n_a(T-1))> \frac{8\beta\log(n_a(T))}{\Delta_{ak}^2}$. Let $j$, $j\in [n_a(T)]$ be the biggest index such that  $T_k(n_a(t_{j-1}))\leq \frac{8\beta\log(n_a(T))}{\Delta_{ak}^2}$. Then at any round $t_i>t_j$, $i\in[n_a(T)]$, using the above lemma, we have 
\[
\mathbb{P}\left(\mathbb{I}\{a^{t_i}=a\land b^{t_i}=k\}\right)\leq 2(i-1){i}^{-\beta}.
\]
We let $p_i=2(i-1){i}^{-\beta}$. 
When $\beta \geq3$,
\[
\sum_{i=1}^{n_a(T)} 2({i}-1){i}^{-\beta}\leq\sum_{i=2}^{n_a(T)} 2{i}^{1-\beta}\leq
2\int_{1}^{+\infty}x^{1-\beta}dx=\frac{2}{\beta-2}.
\]
So 
\[
\sum_{i=1}^{n_a(T)}p_i\leq \frac{2}{\beta-2}.
\]

\[
\sum_{i=1}^{n_a(T)} p_i\left(1-p_i\right)
\leq \sum_{i=1}^{n_a(T)} p_i\leq \frac{2}{\beta-2}\leq 2.
\]

Based on the proof of Bernstein inequality, we can similarly have
\[
\mathbb{P}\left(\sum_{i=j+1}^{n_a(T)}\mathbb{I}\{a^{t_i}=a\land b^{t_i}=k\}>\sum_{i=1}^{n_a(T)}p_i+\epsilon\right)
\leq 2\exp \left(-\frac{{n_a(T)}\cdot\epsilon^2}{2\sigma^2+\frac{2\epsilon}{3}}\right),
\]
where $\sigma^2=\frac{1}{n_a(T)}\sum_{i=1}^{n_a(T)}p_i\left(1-p_i\right)\leq \frac{2}{ n_a(T)}$. Then with probability at least $1-\frac{\delta}{n_f n}$,
\[
\sum_{i=j+1}^{n_a(T)}\mathbb{I}\{a^{t_i}=a\land b^{t_i}=k\}\leq\sum_{i=1}^{n_a(T)}p_i+\mathcal{O}\left(\ln \frac{1}{\delta}+\ln n_f n\right).
\]

By the union bound, for any $a\in\mathcal{A}, \beta \geq3, k\in[n_f]$, $k \neq \text{Br}(a)$, with probability at least $1-\delta$
\begin{equation}
\begin{aligned}
T_k(n_a(T))
&\leq T_k(n_a(t_{j-1}))+1+\sum_{i=j+1}^{n_a(T)}\mathbb{I}\{a^{t_i}=a\land b^{t_i}=k\}\\
&\leq T_k(n_a(t_{j-1}))+1+\sum_{i=1}^{n_a(T)}p_i+\mathcal{O}\left(\ln \frac{1}{\delta}+\ln n_f n\right)\\
&\leq T_k(n_a(t_{j -1}))+1+\frac{2}{\beta-2}+\mathcal{O}\left(\ln \frac{1}{\delta}+\ln n_f n\right)\\
&\leq \frac{8\beta\ln(n_a(T))}{\Delta_{ak}^2}+\frac{\beta}{\beta-2}+\mathcal{O}\left(\ln \frac{1}{\delta}+\ln n_f n\right)\\
&\leq \frac{8\beta\ln(T)}{\varepsilon^2}+\frac{\beta}{\beta-2}+\mathcal{O}\left(\ln \frac{1}{\delta}+\ln n_f n\right).
\end{aligned}
\end{equation}

\subsection{Bound for Term I}
We can see that when $b^t(a^t)=\text{Br}(a^t)$, then $\xi^t(a^t,\text{Br}(a^t))-\xi^t(a^t,b^t(a^t))=0$. Since $\xi^t(a^t,\text{Br}(a^t))\in [0,1]$, so $\left|\xi^t(a^t,\text{Br}(a^t))-\xi^t(a^t,b^t(a^t))\right|\leq 1$. Based on these two arguments, we have
\[
\sum_{t=1}^T \left|\xi^t(a^t,\text{Br}(a^t))-\xi^t(a^t,b^t(a^t))\right|
\leq \sum_{t=1}^T \mathbb{I}\left\{b^t(a^t)\neq \text{Br}(a^t)\right\}
=\sum_{a \in A}\sum_{k\neq \text{Br}(a)}T_k(n_a(T)).
\]
We define 
\[
M_t = \mathbb{E}_{a^t\sim \widetilde{P}^t}\left[ \xi^t(a^t,\text{Br}(a^t))-\xi^t(a^t,b^t(a^t))\right]
-\left[\xi^t(a^t,\text{Br}(a^t))-\xi^t(a^t,b^t(a^t))\right],
\]

\[
-2\leq M_t \leq 2.
\]

We construct the following filtration. For any $t \in [T]$, we define following $\sigma$-algebra as follows
\[
\mathcal{F}^t = \sigma\left(\{a^i, \xi^i(a^i,b^i(a^i)),\xi_f^i(a^i,b^i(a^i))\}_{i\in[t]}\right).
\]
$M_1,M_2,\cdots,M_T$ is a martingale difference sequence with respect to filtration $\{\mathcal{F}^t\}_{t\in [T]}$, which means 
\[
\mathbb{E}[M_{t+1}|\mathcal{F}^t]=0.
\]
By applying Azuma's inequality to the martingale difference sequence, we have
\[
\mathbb{P}[|\sum_{t=1}^{T} M_t|> \varepsilon]\leq 2\exp\left(\frac{-2\varepsilon^2}{16T}\right)=2\exp\left(\frac{-\varepsilon^2}{8T}\right).
\]

So the following inequality holds with probability at least $1-\delta$

\[
\begin{aligned}
&\text{Term I}-\left[\sum_{t=1}^T \xi^t(a^t,\text{Br}(a^t))-\xi^t(a^t,b^t(a^t))\right]\\
\leq&\left|\sum_{t=1}^T\mathbb{E}_{a^t\sim \widetilde{P}^t}\left[ \xi^t(a^t,\text{Br}(a^t))-\xi^t(a^t,b^t(a^t))\right]
-\left[\sum_{t=1}^T \xi^t(a^t,\text{Br}(a^t))-\xi^t(a^t,b^t(a^t))\right]\right|\\
\leq& 2\sqrt{2T\ln{\frac{2}{\delta}}}.
\end{aligned}
\]

Using the results above, we have the following inequality holds with probability at least $1-\delta$
\begin{equation}
\begin{aligned}
&\text{Term I}=\left[\sum_{t=1}^T \xi^t(a^t,\text{Br}(a^t))-\xi^t(a^t,b^t(a^t))\right]
+\text{Term I}-\left[\sum_{t=1}^T \xi^t(a^t,\text{Br}(a^t))-\xi^t(a^t,b^t(a^t))\right]\\
&\leq \sum_{a \in A}\sum_{k\neq \text{Br}(a)}T_k(n_a(T))+2\sqrt{2T\ln{\frac{2}{\delta}}}\\
&\leq \mathcal{O}\left(n_f n \left(\frac{8\beta\log(T)}{\varepsilon^2}+\frac{\beta}{\beta-2}+\ln \frac{1}{\delta}+\ln n_f n\right)+\sqrt{T\ln{\frac{1}{\delta}}}\right).
\end{aligned}
\end{equation}

\subsection{Bound for Term II}
In this part, we focus on the the influence of the additional explicit exploration parameter $\alpha$ to the regret bound.
For simplicity, since we only have one leader, we let $L^t(a)=\xi^t(a,b^t(a))\in [0,1]$, $a\in \mathcal{A}$,
and let $\widetilde{L}^t(a)=\frac{\xi_i^t(a,b^t(a))}{\widetilde{P}^t(a)}\mathbb{I}\{a^t=a\}$. We denote $a_j$ is the $j-$th action of the leader.

As a reminder, $\widetilde{P}_i^t=(1-\alpha)P_i^t+\alpha[1/n,...,1/n]$.
To facilitate our proof, we denote $\widetilde{P}_j^t=\widetilde{P}^t(a_j)$, $P_j^t=P^t(a_j)$ for $j\in [n]$.

We calculate the expectation of $\widetilde{L}^t(a)$, $j\in[n]$:
\[
\begin{aligned}
\mathbb{E} \left[\widetilde{L}^t(a_j)\right]
&=\mathbb{E}_{a_j\sim \widetilde{P}^t} \left[\widetilde{L}^t(a_j)\right]
=\mathbb{E}_{a_j\sim \widetilde{P}^t}\left[\frac{\xi_i^t(a_j,b^t(a_j))}{\widetilde{P}^t(a_j)}\mathbb{I}\{a^t=a_j\}\right]\\
&=\frac{\xi_i^t(a_j,b^t(a_j))}{\widetilde{P}^t(a_j)}\mathbb{E}_{a_j\sim \widetilde{P}^t}\left[\mathbb{I}\{a^t=a_j\}\right]\\
&=\xi^t(a_j,b^t(a_j))
=L^t(a_j).
\end{aligned}
\]
Similarly, the variance of $\widetilde{L}^t(a_j)$, $j\in[n]$:

\[
\mathbb{E} \left[(\widetilde{L}^t(a_{j}))^2\right]
=\mathbb{E}_{a_j\sim \widetilde{P}^t}\left[(\widetilde{L}^t(a_{j}))^2\right]
=\mathbb{E}_{a_j\sim \widetilde{P}^t}\left[\frac{\xi_i^t(a_j,b^t(a_j))^2}{\widetilde{P}^t(a_j)^2}\mathbb{I}\{a^t=a_j\}\right]
=\frac{(L_i^t(a_{j}))^2}{\widetilde{P}_j^t}.
\]

\[
\begin{aligned}
\text{Term II}&= \mathbb{E}\left[\sum_{t=1}^T \xi^t(a^t,b^t(a^t))-\xi^t(a_\star,b^t(a_\star))\right]
= \sum_{t=1}^T\mathbb{E}_{a_j\sim \widetilde{P}^t}\left[ L^t(a^t)-L^t(a_\star)\right]\\
&=\sum_{t=1}^T \mathbb{E}_{a_j\sim \widetilde{P}^t}\left[\sum_{j=1}^n \widetilde{L}^t(a_j)\widetilde{P}_j^t-\widetilde{L}^t(a_\star)\right]\\
&=\sum_{t=1}^T \mathbb{E}_{a_j\sim \widetilde{P}^t}\left[\sum_{j=1}^n \widetilde{L}^t(a_j)\left(\alpha P_j^t+\frac{\alpha}{n}\right)-\widetilde{L}^t(a_\star)\right]\\
&=(1-\alpha)\mathbb{E}\left[\sum_{t=1}^T \sum_{j=1}^n \widetilde{L}^t(a_j)P_j^t-\widetilde{L}^t(a_\star)\right] 
+ \frac{\alpha}{n}\sum_{t=1}^T \mathbb{E}_{a_j\sim \widetilde{P}^t}\left[\sum_{j=1}^n \widetilde{L}^t(a_j)\right]
+\alpha \sum_{t=1}^T \mathbb{E}_{a_j\sim \widetilde{P}^t} \left[\widetilde{L}^t(a_\star)\right]\\
&=(1-\alpha)\mathbb{E}\left[\sum_{t=1}^T \sum_{j=1}^n \widetilde{L}^t(a_j)P_j^t-\widetilde{L}^t(a_\star)\right] 
+\frac{\alpha}{n}\sum_{t=1}^T \sum_{j=1}^n L^t(a_j)
+\alpha \sum_{t=1}^T L^t(a_\star) \\
&=(1-\alpha)\mathbb{E}\left[\sum_{t=1}^T \sum_{j=1}^n \widetilde{L}^t(a_j)P_j^t-\widetilde{L}^t(a_\star)\right] 
+\frac{\alpha}{n}\sum_{t=1}^T \sum_{j=1}^n 1
+\alpha \sum_{t=1}^T 1 \\
&= (1-\alpha)\mathbb{E}\left[\sum_{t=1}^T \sum_{j=1}^n \widetilde{L}^t(a_j)P_j^t-\widetilde{L}^t(a_\star)\right]  
+ \alpha T+\alpha T.
\end{aligned}
\]

We can see that $P^t$ is produced by the Hedge algorithm's updating rule with loss $\widetilde{L}^t(a_j)$ for action $a_j\in \mathcal{A}$ at each round $t$, so based on the regret bound of Hedge algorithm, and $\widetilde{P}_j^t\geq \frac{\alpha}{n}$ we have,

\[
\begin{aligned}
\mathbb{E}&\left[\sum_{j=1}^n \widetilde{L}^t(a_j)P_j^t-\widetilde{L}^t(a_\star)\right]
\leq \frac{\ln n}{\eta}+ \mathbb{E}\left[\eta\sum_{t=1}^T\sum_{j=1}^n P_j^t(\widetilde{L}^t(a_{j}))^2\right]\\
&=\frac{\ln n}{\eta}+ \eta\sum_{t=1}^T\sum_{j=1}^n P_j^t\mathbb{E}\left[(\widetilde{L}^t(a_{j}))^2\right]\\
&=\frac{\ln n}{\eta}+ \eta\sum_{t=1}^T\sum_{j=1}^n P_j^t\frac{(L_i^t(a_{j}))^2}{\widetilde{P}_j^t}\\
&\leq \frac{\ln n}{\eta}+ \eta\sum_{t=1}^T\sum_{j=1}^n P_j^t\frac{1}{\widetilde{P}_j^t}
\leq \frac{\ln n}{\eta}+ \eta\sum_{t=1}^T\sum_{j=1}^n P_j^t \frac{n}{\alpha}
=\frac{\ln n}{\eta}+ \frac{\eta n^2 T}{\alpha}.
\end{aligned}
\]

So we have

\[
\begin{aligned}
\text{Term II}
&\leq (1-\alpha)\mathbb{E}\left[\sum_{t=1}^T \sum_{j=1}^n [\widetilde{L}^t(a_j)P_j^t-\widetilde{L}^t(a_\star)\right]  
+ \alpha T+\alpha T\\
&\leq (1-\alpha)\left[\frac{\ln n}{\eta}+ \frac{\eta n^2 T}{\alpha}\right]+2\alpha T\\
&\leq \frac{\ln n}{\eta}+ \frac{\eta n^2 T}{\alpha}+2\alpha T.
\end{aligned}
\]
Set $\alpha = \mathcal{O}\left({n^{\frac{2}{3}}(\ln n)^{\frac{1}{3}}T^{-\frac{1}{3}}}\right), \eta = \mathcal{O}\left({n^{-\frac{2}{3}}(\ln n)^{\frac{2}{3}}T^{-\frac{2}{3}}}\right)$, we have 
\begin{equation}
\begin{aligned}
\text{Term II}
&\leq \frac{\ln n}{\eta}+ \frac{\eta n^2 T}{\alpha}+2\alpha T
= \mathcal{O}\left(n^\frac{2}{3}(\ln n)^\frac{1}{3}T^\frac{2}{3} \right).
\end{aligned}
\end{equation}

\subsection{Bound for Term III}
For any action $a\in \mathcal{A}$, we assume that action $a$ is chosen at round $t_1,t_2,\cdots,t_{n_a(T)}$. 
For any $a\in A$, let $s_i=t_{i}-t_{i-1}$ represents the length of interval leader chooses action $a$ between the $i$-th time and the $(i+1)$-th time, $i \in [n_a(T)]$, and let $s_{n_a(T)+1}=T-n_a(T)$, $t_0=0$. 

We notice that when $t_{i-1}< t\leq t_{i}$, $t\in [T]$, $b^t(a)=b^{t_{i}}(a)$, so we have the following inequality
\[
\begin{aligned}
\text{Term III}
&=\mathbb{E}\left[\sum_{t=1}^T \xi^t(a,b^t(a))-\xi^t(a,\text{Br}(a))\right]\\
&=\sum_{t=1}^T \xi^t(a,b^t(a))-\xi^t(a,\text{Br}(a))\\
&\leq \sum_{t=1}^T \mathbb{I}\{b^t(a)\neq \text{Br}(a)\}\\
&\leq s_{n_a(T)+1}+\sum_{i=1}^{n_{a}(T)}s_i\mathbb{I}\left[b^{t_i}(a)\neq \text{Br}(a)\right]\\
&\leq \max_{i \in [n_{a}(T)+1]}\{s_i\}\left(1+\sum_{i=1}^{n_{a}(T)}\mathbb{I}\left[b^{t_i}(a)\neq \text{Br}(a)\right]\right)\\
&=s_{\max} \left(1+\sum_{k\neq \text{Br}(a)} T_k(n_a(T))\right)\\
&\leq 2s_{\max}\sum_{k\neq \text{Br}(a)} T_k(n_a(T)).
\end{aligned}
\]

Next we need to bound $s_{\max}$. Since we add an explicit exploration parameter $\alpha$ in the algorithm, the probability of leader choosing any action $a$ is at least $\frac{\alpha}{n}$ at each round, so we have

\[
\mathbb{P}\left(s_{\max}\geq k\right)\leq \left(1-\frac{\alpha}{n}\right)^k,\, Let \; \delta = \left(1-\frac{\alpha}{n}\right)^k.
\]
Noticed that $\ln(1-x)\leq -x$ when $x\in [0,1]$, so we have
\[
\ln \delta = k\ln\left(1-\frac{\alpha}{n}\right)\leq k\left(-\frac{\alpha}{n}\right), \; k \leq \frac{\ln\frac{1}{\delta}}{\frac{\alpha}{n}}=\frac{n\ln\frac{1}{\delta}}{\alpha}.
\]

So with probability at least $1-\delta$,
\[
s_{\max}\leq\frac{n\ln\frac{1}{\delta}}{\alpha}.
\]

We set $\alpha = \mathcal{O}\left({n^{\frac{2}{3}}(\ln n)^{\frac{1}{3}}T^{-\frac{1}{3}}}\right)$. So with probability at least $1-\delta$, $s_{\max}\leq \mathcal{O}\left(n^{\frac{1}{3}}(\ln n)^{-\frac{1}{3}}T^{\frac{1}{3}}\ln \frac{1}{\delta}\right)$.
\begin{equation}
\begin{aligned}
\text{Term III}
&\leq s_{\max}\sum_{k\neq \text{Br}(a)} T_k(n_a(T))\\
&\leq \mathcal{O}\left(n^{\frac{1}{3}}T^{\frac{1}{3}}\ln \frac{1}{\delta}n_f\left(\frac{8\beta\log(T)}{\varepsilon^2}+\frac{\beta}{\beta-2}+\ln \frac{1}{\delta}+\ln n_f n\right)\right).
\end{aligned}
\end{equation}

\subsection{Bound for $R^S(T)$}
Using concentration inequality, the following inequality holds with probability at least $1-\delta$,

\[
\left|R^S(T)-\widetilde{R}^S(T)\right|\leq \sqrt{2T\ln{\frac{2}{\delta}}}.
\]

For a sufficient small $p$,  $\ln\frac{1}{p}\leq(\ln\frac{1}{p})^2$. Set $\alpha = \mathcal{O}\left({n^{\frac{2}{3}}(\ln n)^{\frac{1}{3}}T^{-\frac{1}{3}}}\right), \eta = \mathcal{O}\left({n^{-\frac{2}{3}}(\ln n)^{\frac{2}{3}}T^{-\frac{2}{3}}}\right)$, by the union bound, the following inequality holds with probability at least $1-p$, $p=4\delta$,
\[
\begin{aligned}
R^S(T)&=\widetilde{R^S}(T)+R^S(T)-\widetilde{R}^S(T)\leq \widetilde{R^S}(T)+\mathcal{O}\left(\sqrt{T\ln\frac{1}{p}}\right)\\
&=\text{Term I}+\text{Term II}+\text{Term III}+\mathcal{O}\left(\sqrt{T\ln\frac{1}{p}}\right)\\
&\leq \mathcal{O}\left(n_f n \left(\frac{8\beta\log(T)}{\varepsilon^2}+\frac{\beta}{\beta-2}+\ln \frac{1}{p}+\ln n_f n\right)+\sqrt{T\ln{\frac{1}{p}}}\right)\\
&+\mathcal{O}\left(n^{\frac{1}{3}}T^{\frac{1}{3}}\ln \frac{1}{p}n_f\left(\frac{8\beta\log(T)}{\varepsilon^2}+\frac{\beta}{\beta-2}+\ln \frac{1}{p}+\ln n_f n\right)\right)\\
&+\mathcal{O}\left(n^\frac{2}{3}(\ln n)^\frac{1}{3}T^\frac{2}{3} \right)+\mathcal{O}\left(\sqrt{T\ln\frac{1}{p}}\right)\\
&\leq \widetilde{\mathcal{O}}\left(n^{\frac{1}{3}}n_f\frac{\beta}{\varepsilon^2}T^{\frac{1}{3}}(\ln\frac{1}{p})^2+n^\frac{2}{3}(\ln n)^\frac{1}{3}T^\frac{2}{3}+n_f n \frac{\beta}{\beta-2}\right).
\end{aligned}
\]

\section{Proof of Theorem 2}\label{appendix:theorem2}
Let $\epsilon^T=R^S\left(T\right)/T$, with probability at least $1-p$, we have
\[
\frac{1}{T}\sum_{t=1}^T\mathbb{E}_{a\sim \widetilde{P}^t}\left[l_i(a^t,\text{Br}(a^t))\right]
\leq l(a_\star,\text{Br}(a_\star))+\epsilon^T.
\]

Since expectations are linear, we can
rewrite the inequality using the time averaged action profile distribution $\bar{P}=\frac{1}{T}\sum_{t=1}^T \widetilde{P}^t$ as follows
\[
\mathbb{E}_{a \sim \Bar{P}}\left[l(a,\text{Br}(a)) \right] \leq l(a_\star,\text{Br}(a_\star))  + \epsilon^T.
\]
We can use the following process to sample an action from distribution $\bar{P}$: first we sample $t$ uniformly from $[T]$, then we sample an action $a$ from the distribution $\widetilde{P}^t$.

\section{Proof of Theorem 3}\label{appendix:theorem3}

For the case of $m=2$, we set $\alpha = \mathcal{O}\left({n^{\frac{2}{3}}(\ln n)^{\frac{1}{3}}T^{-\frac{1}{3}}}\right), \eta = \mathcal{O}\left({n^{-\frac{2}{3}}(\ln n)^{\frac{2}{3}}T^{-\frac{2}{3}}}\right)$. Combined with the techniques used in the proof of Proposition 2 and Theorem 1, we have for $i=1,2$, with probability at least $1-p$,
\[
\begin{aligned}
R_i^S(T)
&\leq \widetilde{\mathcal{O}}\left(\frac{n^2\ln\frac{1}{p}}{\alpha^2}n_f\left(\frac{\beta}{\varepsilon^2}+\frac{\beta}{\beta-2}+\ln \frac{1}{p}\right)+\frac{\ln n}{\eta}+ \frac{\eta n^2 T}{\alpha}+2\alpha T+n_f n \frac{\beta}{\beta-2}\right)\\
&\leq\widetilde{\mathcal{O}}\left(n_f n^{\frac{2}{3}}\frac{\beta}{\varepsilon^2}T^\frac{2}{3}(\ln \frac{1}{p})^2+n_f n \frac{\beta}{\beta-2}\right).
\end{aligned}
\]

For the general $m$, we define $\widetilde{\chi}^t=\widetilde{P}_1^t\times\cdots\times \widetilde{P}_m^t$ , and we decompose the noisy regret as follows

\[
\begin{aligned}
\widetilde{R_i^S}(T)&=\sum_{t=1}^T\mathbb{E}_{a^t\sim \widetilde{\chi}^t}\left[ \xi_i^t(a^t,\text{Br}(a^t))-\xi_i^t(a_{i,\star},a_{-i}^t\text{Br}(a_{i,\star},a_{-i}^t))\right] \\
&=\underbrace{\sum_{t=1}^T\mathbb{E}_{a^t\sim \widetilde{\chi}^t}\left[ \xi_i^t(a^t,\text{Br}(a^t))\right]-\mathbb{E}_{a_{i}^t\sim \widetilde{P}_{i}^t}\left[\xi_i^t(a^t,\text{Br}(a^t))\right]}_{\text{Term I}}\\
&+\underbrace{\sum_{t=1}^T\mathbb{E}_{a_{i}^t\sim \widetilde{P}_i^t}\left[ \xi^t(a^t,\text{Br}(a^t))-\xi^t(a^t,b^t(a^t))\right]}_{\text{Term II}}\\
&+\underbrace{\sum_{t=1}^T\mathbb{E}_{a_{i}^t\sim \widetilde{P}_i^t}\left[ \xi^t(a^t,b^t(a^t))-\xi^t(a_{i,\star},a_{-i}^t,b^t(a_{i,\star},a_{-i}^t))\right]}_{\text{Term III}}\\ 
&+\underbrace{\sum_{t=1}^T\mathbb{E}_{a_{i}^t\sim \widetilde{P}_i^t}\left[ \xi^t(a_{i,\star},a_{-i}^t,b^t(a_{i,\star},a_{-i}^t))-\xi^t(a_{i,\star},a_{-i}^t,\text{Br}(a_{i,\star},a_{-i}^t))\right]}_{\text{Term IV}}\\
&+\underbrace{\sum_{t=1}^T\xi_i^t(a_{i,\star},a_{-i}^t,\text{Br}(a_{i,\star},a_{-i}^t))-\mathbb{E}_{a^t\sim \widetilde{\chi}^t}\left[\xi_i^t(a_{i,\star},a_{-i}^t,\text{Br},(a_{i,\star},a_{-i}^t))\right]}_{\text{Term V}}.
\end{aligned}
\]
\text{Term I} and \text{Term V} can be bounded using similar techniques of the proof of Proposition 2. \text{Term II}, \text{Term III} and \text{Term IV} can be bounded using similar techniques of the proof of Theorem 1. We set $\alpha= \mathcal{O}\left(nT^{-\frac{1}{m+1}}\right), \eta = \mathcal{O}\left(\sqrt{T^{-\frac{m+2}{m+1}}n\ln n}\right)$, and $T\geq \mathcal{O}\left(n^{m+1}\right)$. Combined with the techniques used in the proof of Proposition 2 and Theorem 1, with probability at least $1-p$, we have
\[
\begin{aligned}
R_i^S(T)
&\leq \widetilde{\mathcal{O}}\left(\frac{n^m\ln\frac{1}{p}}{\alpha^m}n_f\left(\frac{\beta}{\varepsilon^2}+\frac{\beta}{\beta-2}+\ln \frac{1}{p}\right)+\frac{\ln n}{\eta}+ \frac{\eta n^2 T}{\alpha}+2\alpha T+n_f n \frac{\beta}{\beta-2}\right)\\
&\leq\widetilde{\mathcal{O}}\left(\left(n+n_f \frac{\beta}{\varepsilon^2}(\ln \frac{1}{p})^2\right)T^{\frac{m}{m+1}}+n_f n \frac{\beta}{\beta-2}\right).
\end{aligned}
\]

\section{Proof of Theorem 4}\label{appendix:theorem4}

Based on leaders' pure exploration strategy, at every round $t\leq t_0$, for every $a\in\mathcal{A}$: $\mathbb{P}\left(\mathbb{I}\{a^t=a\}\right)=\frac{1}{n^m}$. Let $p_0=\frac{1}{n^m}$. Using Bernstein inequality, we have

\[
\mathbb{P}\left(\frac{1}{t_0}\sum_{t=1}^{t_0} \mathbb{I}\{a^t=a\}-p_0<-\epsilon\right)\leq \exp \left(-\frac{t_0\epsilon^2}{2\sigma^2+\frac{2\epsilon}{3}}\right),
\]
where $\sigma^2=p_0(1-p_0)<p_0$. Let $\epsilon=\frac{p_0}{2}$, we have
\[
\mathbb{P}\left(\sum_{t=1}^{t_0} \mathbb{I}\{a^t=a\}<t_0 \frac{p_0}{2}
\right)\leq \exp \left(-\frac{t_0(\frac{p_0}{2})^2}{2\sigma^2+\frac{2}{3}\frac{p_0}{2}}\right)\leq \exp \left(-\frac{t_0(\frac{p_0}{2})^2}{2p_0+\frac{2}{3}\frac{p_0}{2}}\right)=\exp\left(-\frac{3}{28}t_0 p_0\right).
\]

Let $t_0\geq \max\{\frac{4q}{p_0}, \frac{28}{3p_0}\ln\frac{2}{p_0 p}\}$, with probability at least $1-\frac{p_0 p}{2}$, we have
\[
T_a=\sum_{t=1}^{t_0} \mathbb{I}\{a^t=a\}\geq 2q.
\]
\begin{lemma} \citep{audibert2010best} 
If UCB-E is run with parameter $0<e\leq \frac{25}{36}\frac{T_a - n_f}{H_a}$, let $T_a$ be the number of times leaders choose $a$ in the first $t_0$ rounds, $H_a=\sum_{k=1}^{n_f} \frac{1}{\Delta_{ak}^2}$, then it satisfies
\[
\mathbb{P}\left(\widehat{\text{Br}}(a)\neq \text{Br}(a)\right)\leq 2T_a n_f \exp\left(-\frac{2e}{25}\right).
\]
In particular, when we set $e=\frac{25}{36}\frac{T_a - n_f}{H_a}$, we have $\mathbb{P}\left(\widehat{\text{Br}}(a)\neq \text{Br}(a) \right)\leq 2T_a n_f \exp\left(-\frac{T_a - n_f}{18H_a}\right)$.
\end{lemma}
\begin{proof}
See Theorem 1 of \citet{audibert2010best} for a detailed proof.
\end{proof}

We set a $q\geq 18H_a\left(\ln\frac{2q n_f}{p}+m\ln n\right)+n_f$. So $q> \ln \frac{2}{p_0 p}$. Since $t_0\geq \max\{\frac{4q}{p_0}, \frac{28}{3p_0}\ln\frac{2}{p_0 p}\}$, so we can set $t_0\geq\frac{28q}{3p_0}=\frac{28qn^m}{3}=\mathcal{O}\left(n^m q\right)$. 

Base on this lemma, when $T_a\geq 2q\geq 2\left(18H_a\left(\ln\frac{2q n_f}{p_0 p}\right)+n_f\right)\geq18H_a\left(\ln\frac{4q n_f}{p_0 p}\right)+n_f$,

\[
\mathbb{P}\left(\widehat{\text{Br}}(a)\neq \text{Br}(a) \right)\leq \frac{p_0 p}{2}.
\]

By the union bound, with probability at least $1-p_0 p$, we have $\widehat{\text{Br}}(a)= \text{Br}(a)$.

Similarly, with probability at least $1-n^m p_0 p=1-p$, for any $a\in\mathcal{A}$, we have
\[
\widehat{\text{Br}}(a)= \text{Br}(a).
\]

After the follower commits its best response predictors, the rest $T-t_0$ rounds game for the leader is actually a semi-bandit MLSF game, so we use the result of Proposition 2. So by the union bound, with probability at least $1-2p$,
\[
\begin{aligned}
R_{i}^S\left(T\right)
  &\leq\mathcal{O}\left(t_0+\sqrt{n(T-t_0)\ln n}+\sqrt{(T-t_0)\ln{\frac{1}{p}}}\right)\\
  &\leq\mathcal{O}\left(t_0+\sqrt{nT\ln n}+\sqrt{T\ln{\frac{1}{p}}}\right).
\end{aligned}
\]

Using reduction from no-external to no-swap regret, which is Theorem 2 of~\citet{ito2020tight}, we get an $\epsilon^T$-CSE for leaders in MLSF bandit game, where $\epsilon^T=\mathcal{O}\left(\frac{t_0}{T}+\sqrt{\frac{1}{T}n^2\ln n}+\sqrt{\frac{1}{T}n\ln \frac{1}{p}}\right)$.

\end{document}